%% file: choicerank.tex
\documentclass[12pt,a4paper,oneside]{article}
\input{preamble}

\title{ChoiceRank: Identifying Preferences \\
from Node Traffic in Networks}
\author{
Lucas Maystre\thanks{School of Computer and Communication Sciences, EPFL, Switzerland.}\\
\email{lucas.maystre@epfl.ch}
\and
Matthias Grossglauser\footnotemark[1]\\
\email{matthias.grossglauser@epfl.ch}
}

\begin{document}
\maketitle

\begin{abstract}
Understanding how users navigate in a network is of high interest in many applications.
We consider a setting where only aggregate node-level traffic is observed and tackle the task of learning edge transition probabilities.
We cast it as a preference learning problem, and we study a model where choices follow Luce's axiom.
In this case, the $O(n)$ marginal counts of node visits are a sufficient statistic for the $O(n^2)$ transition probabilities.
We show how to make the inference problem well-posed regardless of the network's structure, and we present ChoiceRank, an iterative algorithm that scales to networks that contains billions of nodes and edges.
We apply the model to two clickstream datasets and show that it successfully recovers the transition probabilities using only the network structure and marginal (node-level) traffic data.
Finally, we also consider an application to mobility networks and apply the model to one year of rides on New York City's bicycle-sharing system.
\end{abstract}

\input{01-introduction}
\input{02-model}
\input{03-relwork}
\input{04-theory}
\input{05-algorithm}
\input{06-experiments}
\input{07-conclusion}

\appendix
\input{0A-extensions}
\input{0B-maxlik}
\input{0C-algorithm}

\hyphenation{wahrscheinlichkeits-rechnung} % :-)
\bibliography{choicerank}

\end{document}

%% file: preamble.tex
\usepackage{fullpage}
\usepackage[english]{babel} % English language/hyphenation.
\usepackage[T1]{fontenc} % Use 8-bit output encoding.
\usepackage[utf8]{inputenc} % Can use UTF-8 in the source files.
\usepackage[babel]{microtype} % Improves appearance of text.

\usepackage{url}
\usepackage{csquotes}
\usepackage{amsmath,amsthm,amssymb,bm,mathtools}
\usepackage[pdftex]{graphicx}
\graphicspath{{fig/}}
\usepackage{enumitem}
\usepackage{siunitx}
\usepackage[style=base]{caption}
\usepackage{subcaption}

\usepackage{algorithm}
\usepackage{algpseudocode}

\usepackage{natbib}
\usepackage{hyperref}

\algnewcommand{\OneLineFor}[1]{\State \algorithmicfor\ #1\ \algorithmicdo}

\newlist{enuminline}{enumerate*}{1}
\newlist{conditions}{enumerate}{1}
\setlist[enuminline,1]{label=\itshape\alph*\upshape)}
\setlist*[conditions,1]{label=C\arabic*.,ref=C\arabic*}

\newcommand{\transp}{\top}
\newcommand{\email}[1]{\href{mailto:#1}{\nolinkurl{#1}}}
\DeclareMathOperator*{\argmax}{arg\,max}

\theoremstyle{plain}
\newtheorem{theorem}{Theorem}

\theoremstyle{definition}
\newtheorem*{definition}{Definition}

%% file: 01-introduction.tex
%%%%%%%%%%%%%%%%%%%%%%%%%%%%%%%%%%%%%%%%%%%%%%%%%%%%%%%%%%%%%%%%%%%%%%%%%
\section{Introduction}  %%%%%%%%%%%%%%%%%%%%%%%%%%%%%%%%%%%%%%%%%%%%%%%%%
\label{sec:introduction}

Consider the problem of estimating click probabilities for links between pages of a website, given a hyperlink graph and aggregate statistics on the number of times each page has been visited.
Naively, one might expect that the probability of clicking on a particular link should be roughly proportional to the traffic of the link's target.
%At first, one might expect that links pointing to high-traffic web pages must be more likely to be clicked.
However, this neglects important structural effects:
a page's traffic is influenced by
\begin{enuminline}
\item the number of incoming links,
\item the traffic at the pages that link to it, and
\item the traffic absorbed by competing links.
\end{enuminline}
In order to successfully infer click probabilities, it is therefore necessary to disentangle the \emph{preference} for a page (i.e., the intrinsic propensity of a user to click on a link pointing to it) from the page's \emph{visibility} (the exposure it gets from pages linking to it).
Building upon recent work by \citet{kumar2015inverting}, we present a statistical framework that tackles a general formulation of the problem:
given a network (representing possible transitions between nodes) and the marginal traffic at each node, recover the transition probabilities.
This problem is relevant to a number of scenarios (in social, information or transportation networks) where transition data is not available due to, e.g., privacy concerns or monitoring costs.

%Our starting point is a choice model, i.e., a model that explains and predicts outcomes of comparisons between two or more alternatives from a universe of $n$ items.
%In this context, a prominent paradigm (which dates back to \citet{thurstone1927method} and \citet{zermelo1928berechnung} almost a century ago) postulates that each item can be characterized by a latent \emph{strength} parameter, and that the (stochastic) comparison outcomes tend to favor items with greater strengths.
%The parameters can be learned from observed choices and used for predicting future comparisons.
%Models based on this paradigm have been successfully applied to problems ranging from ranking chess players based on game outcomes \citep{zermelo1928berechnung, elo1978rating} to understanding consumer behavior based on discrete choices \citep{mcfadden1973conditional}, and to discriminating among multiple classes based on the output of pairwise classifiers \citep{hastie1998classification}.
%Here, we assume that the comparisons are generated from a specific stochastic process on a network, reminiscent of the random-surfer model introduced by \citet{brin1998anatomy}.
%In our setting, every node in the network represents an item, and comparisons take place over the $n$ sets of alternatives induced by the neighborhoods of each node.
We begin by postulating the following model of traffic.
Users navigate from node to node along the edges of the network by making a choice between adjacent nodes at each step, reminiscent of the random-surfer model introduced by \citet{brin1998anatomy}.
Choices are assumed to be independent and generated according to Luce's model \citep{luce1959individual}: each node in the network is chararacterized by a latent \emph{strength} parameter, and (stochastic) choice outcomes tend to favor nodes with greater strengths.
In this model, estimating the transition probabilities amounts to estimating the strength parameters.
Unlike the setting in which choice models are traditionally studied \citep{train2009discrete, maystre2015fast, vojnovic2016parameter}, we do not observe distinct choices among well-identified sets of alternatives.
Instead, we only have access to aggregate, marginal statistics about the traffic at each node in the network.
In this setting, we make the following contributions.

\begin{enumerate}
\item We observe that marginal per-node traffic is a sufficient statistic for the strength parameters.
That is, the parameters can be inferred from marginal traffic data without any loss of information.

\item We show that if the parameters are endowed with a prior distribution, the inference problem becomes well-posed regardless of the network structure.
This is a crucial step in making the framework applicable to real-world datasets.

\item We show that model inference can scale to very large datasets.
We present an iterative EM-type inference algorithm that enables a remarkably efficient implementation---each iteration requires the computational equivalent of two iterations of PageRank.
\end{enumerate}

%To summarize our contributions from a machine-learning practitioner's perspective, we propose a scalable method that is able to learn transition probabilities in a network, given only the network's structure and the marginal traffic at each node.

We evaluate two aspects of our framework using real-world networks.
We begin by demonstrating that local preferences can indeed be inferred from global traffic: we investigate the accuracy of the transition probabilities recovered by our model on three datasets for which we have ground-truth transition data.
First, we consider two hyperlink graphs, representing the English Wikipedia (over two million nodes) and a Hungarian news portal (approximately \num{40000} nodes), respectively.
We model clickstream data as a sequence of independent choices over the links available at each page.
Given only the structure of the graph and the marginal traffic at every node, we estimate the number of transitions between nodes, and we find that our estimate matches ground-truth edge-level transitions accurately in both instances.
Second, we consider the network of New York City's bicycle-sharing service.
For a given ride, given a pick-up station, we model the drop-off station as a choice out of a set of locations.
Our model yields promising results, suggesting that our method can be useful beyond clickstream data.
Next, we test the scalability of the inference algorithm.
We show that the algorithm is able to process a snapshot of the WWW hyperlink graph containing over a hundred billion edges using a single machine.

\paragraph{Organization of the paper.}
In Section~\ref{sec:model}, we formalize the network choice model.
In Section~\ref{sec:relwork}, we briefly review related literature.
In Section~\ref{sec:theory}, we present salient statistical properties of the model and its maximum-likelihood estimator, and we propose a prior distribution that makes the inference problem well-posed.
In Section~\ref{sec:algorithm}, we describe an inference algorithm that enables an efficient implementation.
We evaluate the model and the inference algorithm in Section~\ref{sec:experiments}, before concluding in Section~\ref{sec:conclusion}.
In the appendices, we provide a more in-depth discussion of our model and algorithm, and we present proofs for all the theorems stated in the main text.

%% file: 02-model.tex
%%%%%%%%%%%%%%%%%%%%%%%%%%%%%%%%%%%%%%%%%%%%%%%%%%%%%%%%%%%%%%%%%%%%%%%%%
\section{Network Choice Model}  %%%%%%%%%%%%%%%%%%%%%%%%%%%%%%%%%%%%%%%%%
\label{sec:model}

Let $G = (V,E)$ be a directed graph on $n$ nodes (corresponding to items) and $m$ edges.
We denote the out-neighborhood of node $i$ by $N^+_i$ and its in-neighborhood by $N^-_i$.
We consider the following choice process on $G$.
A user starts at a node $i$ and is faced with alternatives $N^+_i$.
The user chooses item $j$ and moves to the corresponding node.
At node $j$, the user is faced with alternatives $N^+_j$ and chooses $k$, and so on.
At any time, the user can stop.
Figure~\ref{fig:samplenet} gives an example of a graph and the alternatives available at a step of the process.

\begin{figure}[t]
  \centering
  \includegraphics[width=.5\linewidth]{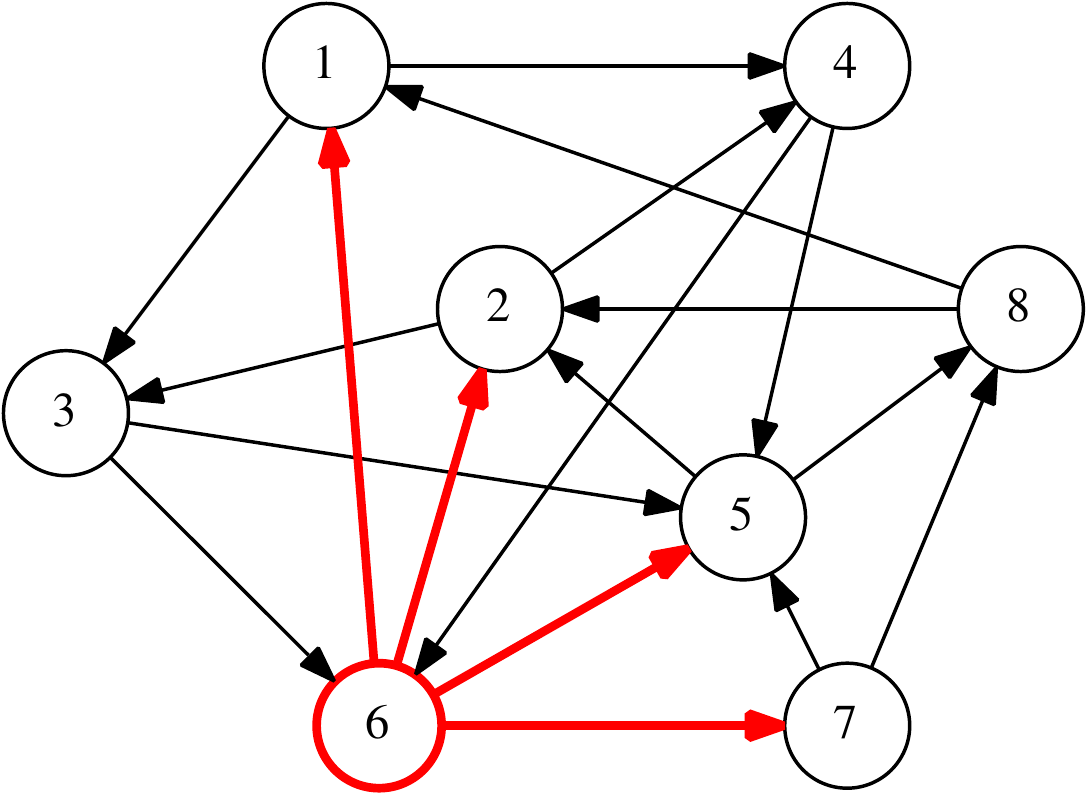}
  \caption{An illustration of one step of the process.
  The user is at node 6 and can reach nodes $N^+_6 = \{1, 2, 5, 7\}$.}
  \label{fig:samplenet}
\end{figure}

To define the transition probabilities, we posit Luce's well-known choice axiom that states that the odds of choosing item $j$ over item $j'$ do not depend on the rest of the alternatives \citep{luce1959individual}.
This axiom leads to a unique probabilistic model of choice.
For every node $i$ and every $j \in N^+_i$, the probability that $j$ is selected among alternatives $N^+_i$ can be written as
\begin{align}
\label{eq:singlelik}
p_{ij} = \frac{\lambda_j}{\sum_{k \in N^+_i} \lambda_k}
\end{align}
for some parameter vector $\bm{\lambda} = \begin{bmatrix}\lambda_1 & \cdots & \lambda_n \end{bmatrix}^\transp \in \mathbf{R}_{>0}^n$.
Intuitively, the parameter $\lambda_i$ can be interpreted as the strength (or utility) of item $i$.
Note that $p_{ij}$ depends only on the out-neighborhood of node $i$.
As such, the choice process satisfies the Markov property, and we can think of the sequence of choices as a trajectory in a Markov chain.
In the context of this model, we can formulate the inference problem as follows.
Given a directed graph $G = (V, E)$ and data on the aggregate traffic at each node, find a parameter vector $\bm{\lambda}$ that fits the data.

%% file: 03-relwork.tex
%%%%%%%%%%%%%%%%%%%%%%%%%%%%%%%%%%%%%%%%%%%%%%%%%%%%%%%%%%%%%%%%%%%%%%%%%
\section{Related Work}  %%%%%%%%%%%%%%%%%%%%%%%%%%%%%%%%%%%%%%%%%%%%%%%%%
\label{sec:relwork}

A variant of the network choice model was recently introduced by \citet{kumar2015inverting}, in an article that lays much of the groundwork for the present paper.
Their generative model of traffic and the parametrization of transition probabilities based on Luce's axiom form the basis of our work.
\citeauthor{kumar2015inverting} define the \emph{steady-state inversion} problem as follows:
Given a graph $G$ and a target stationary distribution, find transition probabilities that lead to the desired stationary distribution.
This problem formulation assumes that $G$ satisfies restrictive structural properties (strong-connectedness, aperiodicity) and is valid only asymptotically, when the sequences of choices made by users are very long.
Our formulation is, in contrast, more general.
In particular, we eliminate any assumptions about the structure of $G$ and cope with finite data in a principled way---in fact, our derivations are valid for choice sequences of any length.
One of our contributions is to explain the steady-state inversion problem in terms of (asymptotic) maximum-likelihood inference in the network choice model.
Furthermore, the statistical viewpoint that we develop also leads to
\begin{enuminline}
\item a robust regularization scheme, and
\item a simple and efficient EM-type inference algorithm.
\end{enuminline}
These important extensions make the model easier to apply to real-world data.

\paragraph{Luce's choice axiom.}
The general problem of estimating parameters of models based on Luce's axiom has received considerable attention.
Several decades before Luce's seminal book \citep{luce1959individual}, \citet{zermelo1928berechnung} proposed a model and an algorithm that estimates the strengths of chess players based on pairwise comparison outcomes (his model would later be rediscovered by \citet{bradley1952rank}).
More recently, \citet{hunter2004mm} explained \citeauthor{zermelo1928berechnung}'s algorithm from the perspective of the minorization-maximization (MM) method.
This method is easily generalized to other models that are based on Luce's axiom, and it yields simple, provably convergent algorithms for maximum-likelihood (ML) or maximum-a-posteriori point estimates.
\citet{caron2012efficient} observe that these MM algorithms can be further recast as expectation-maximization (EM) algorithms by introducing suitable latent variables.
They use this observation to derive Gibbs samplers for a wide family of models.
We take advantage of this long line of work in Section~\ref{sec:algorithm} when developing an inference algorithm for the network choice model.
In recent years, several authors have also analyzed the sample complexity of the ML estimate in Luce's choice model \citep{hajek2014minimax, vojnovic2016parameter} and investigated alternative spectral inference methods \citep{negahban2012iterative, azari2013generalized, maystre2015fast}.
Some of these results could be applied to our setting, but in general they require observing choices among well-identified sets of alternatives.
Finally, we note that models based on Luce's axiom have been successfully applied to problems ranging from ranking players based on game outcomes \citep{zermelo1928berechnung, elo1978rating} to understanding consumer behavior based on discrete choices \citep{mcfadden1973conditional}, and to discriminating among multiple classes based on the output of pairwise classifiers \citep{hastie1998classification}.

\paragraph{Network analysis.}
Understanding the preferences of users in networks is of significant interest in many domains.
For brevity, we focus on literature related to hyperlink graphs.
A method that has undoubtedly had a tremendous impact in this context is PageRank \citep{brin1998anatomy}.
PageRank computes a set of scores that are proportional to the amount of time a surfer, who clicks on links randomly and uniformly, spends at each node.
These scores are based only on the structure of the graph.
The network choice model presented in this paper appears similar at first, but tackles a different problem.
In addition to the structure of the graph, it uses the traffic at each page, and computes a set of scores that reflect the (non-uniform) probability of clicking on each link.
Nevertheless, there are striking similarities in the implementation of the respective inference algorithms (see Section~\ref{sec:experiments}).
The HOTness method proposed by \citet{tomlin2003new} is somewhat related, but tries to tackle a harder problem.
It attempts to estimate jointly the traffic and the probability of clicking on each link, by using a maximum-entropy approach.
At the other end of the spectrum, BrowseRank \citep{liu2008browserank} uses detailed data collected in users' browsers to improve on PageRank.
Our method uses only marginal traffic data that can be obtained without tracking users.
%In the context of mobility analysis, we mention that \citet{ashbrook2003using} and \citet{kafsi2015traveling}

%% file: 04-theory.tex
%%%%%%%%%%%%%%%%%%%%%%%%%%%%%%%%%%%%%%%%%%%%%%%%%%%%%%%%%%%%%%%%%%%%%%%%%
\section{Statistical Properties}  %%%%%%%%%%%%%%%%%%%%%%%%%%%%%%%%%%%%%%%
\label{sec:theory}

In this section, we describe some important statistical properties of the network choice model.
We begin by observing that $O(n)$ values summarizing the traffic at each node is a sufficient statistic for the $O(n^2)$ entries of the Markov-chain transition matrix.
We then connect our statistical model to the steady-state inversion problem defined by \citet{kumar2015inverting}.
Guided by this connection, we study the maximum-likelihood (ML) estimate of model parameters, but find that the estimate is likely to be ill-defined in many scenarios of practical interest.
Lastly, we study how to overcome this issue by introducing a prior distribution on the parameters $\bm{\lambda}$; the prior guarantees that the inference problem is well-posed.

For simplicity of exposition, we present our results for Luce's standard choice model defined in~\eqref{eq:singlelik}.
Our developments extend to the model variant proposed by \citet{kumar2015inverting}, where choice probabilities can be modulated by edge weights.
In Appendix~\ref{app:extensions}, we describe this variant and give the necessary adjustments to our developments.

%%%%%%%%%%%%%%%%%%%%%%%%%%%%%%%%%%%%%%%%%%%%%%%%%%%%%%%%%%%%%%%%%%%%%%%%%
\subsection{Aggregate Traffic Is a Sufficient Statistic}

Let $c_{ij}$ denote the number of transitions that occurred along edge $(i, j) \in E$.
Starting from the transition probability defined in~\eqref{eq:singlelik}, we can write the log-likelihood of $\bm{\lambda}$ given data $\mathcal{D} = \{ c_{ij} \mid (i, j) \in E \}$ as
\begin{align}
\ell(\bm{\lambda} ; \mathcal{D})
    &= \sum_{(i,j) \in E} c_{ij} \bigg[ \log \lambda_j - \log \sum_{k \in N^+_i} \lambda_k \bigg] \nonumber \\
    &= \sum_{j = 1}^n \sum_{i \in N^-_j}\!c_{ij} \log \lambda_j
       - \sum_{i = 1}^n \sum_{j \in N^+_i}\!c_{ij} \log \sum_{k \in N^+_i} \lambda_k \nonumber \\
    &= \sum_{i = 1}^n \bigg[ c^-_i \log \lambda_i - c^+_i \log \sum_{k \in N^+_i} \lambda_k \bigg], \label{eq:loglik}
\end{align}
where $c^-_i = \sum_{j \in N^-_i} c_{ji}$ and $c^+_i = \sum_{j \in N^+_i} c_{ij}$ is the aggregate number of transitions arriving in and originating from $i$, respectively.
This formulation of the log-likelihood exhibits a key feature of the model:
the set of $2n$ counts $\{ (c^-_i, c^+_i) \mid i \in V \}$ is a sufficient statistic of the $O(n^2)$ counts $\{ c_{ij} \mid (i, j) \in E \}$ for the parameters $\bm{\lambda}$.
(In Appendix~\ref{app:extensions}, we show that it is in fact minimally sufficient.)
In other words, it is enough to observe marginal information about the number of arrivals and departures at each node---we collectively call this data the \emph{traffic} at a node---and no additional information can be gained by observing the full choice process.
This makes the model particularly attractive, because it means that it is unnecessary to track users across nodes.
In several applications of practical interest, tracking users is undesirable, difficult, or outright impossible, due to
\begin{enuminline}
\item privacy reasons,
\item monitoring costs, or
\item lack of data in existing datasets.
\end{enuminline}

Note that if we make the additional assumption that the flow in the network is conserved, then $c^-_i = c^+_i$.
If users' typical trajectories consist of many hops, it is reasonable to approximate $c^-_i$ or $c^+_i$ using that assumption, should one of the two quantities be missing.

%%%%%%%%%%%%%%%%%%%%%%%%%%%%%%%%%%%%%%%%%%%%%%%%%%%%%%%%%%%%%%%%%%%%%%%%%
\subsection{Connection to the Steady-State Inversion Problem}
% Also works if all the Markov trajectories are loops.

In recent work, \citet{kumar2015inverting} define the problem of \emph{steady-state inversion} as follows:
Given a strongly-connected directed graph $G = (V, E)$ and a target distribution over the nodes $\bm{\pi}$, find a Markov chain on $G$ with stationary distribution $\bm{\pi}$.
As there are $m = O(n^2)$ degrees of freedom (the transition probabilities) for $n$ constraints (the stationary distribution), the problem is in most cases underdetermined.
Following Luce's ideas, the transition probabilities are constrained to be proportional to a latent score of the destination node as per \eqref{eq:singlelik}, thus reducing the number of parameters from $m$ to $n$.
Denote by $P(\bm{s})$ the Markov-chain transition matrix parametrized with scores $\bm{s}$.
The score vector $\bm{s}$ is a solution for the steady-state inversion problem if and only if $\bm{\pi} = \bm{\pi} P(\bm{s})$, or equivalently
\begin{align}
\label{eq:balance}
\pi_i = \sum_{j \in N^-_i} \frac{s_i}{\sum_{k \in N^+_j} s_k} \pi_j \quad \forall i.
\end{align}
In order to formalize the connection between \citeauthor{kumar2015inverting}'s work and ours, we now express the steady-state inversion problem as that of asymptotic maximum-likelihood estimation in the network choice model.
Suppose that we observe node-level traffic data $\mathcal{D} = \{ (c^-_i, c^+_i) \mid i \in V \}$ about a trajectory of length $T$ starting at an arbitrary node.
We want to obtain an estimate of the parameters $\bm{\lambda}^\star$ by maximizing the average log-likelihood $\hat{\ell}(\bm{\lambda}) = \frac{1}{T} \ell (\bm{\lambda} ; \mathcal{D})$.
From standard convergence results for Markov chains \citep{kemeny1976finite}, it follows that as $G$ is strongly connected, $\lim_{T \to \infty} c^-_i / T = \lim_{T \to \infty} c^+_i / T = \pi_i$.
Therefore,
\begin{align*}
\hat{\ell}(\bm{\lambda})
    = \sum_{i = 1}^n \bigg[ \frac{c^-_i}{T} \log \lambda_i - \frac{c^+_i}{T} \log \sum_{k \in N^+_i} \lambda_k \bigg]
    \xrightarrow{T \to \infty} \sum_{i = 1}^n \pi_i \bigg[ \log \lambda_i - \log \sum_{k \in N^+_i} \lambda_k \bigg].
\end{align*}
Let $\bm{\lambda}^\star$ be a maximizer of the average log-likelihood.
When $T \to \infty$, the optimality condition $\nabla \hat{\ell} (\bm{\lambda}^\star) = \bm{0}$ implies
\begin{align}
\frac{\partial \hat{\ell}(\bm{\lambda})}{\partial \lambda_i} \bigg|_{\bm{\lambda} = \bm{\lambda}^\star}
    = \frac{\pi_i}{\lambda^\star_i} - \sum_{j \in N^-_i} \frac{\pi_j}{\sum_{k \in N^+_j} \lambda^\star_k} = 0
    \iff \pi_i &= \sum_{j \in N^-_i} \frac{\lambda^\star_i}{\sum_{k \in N^+_j} \lambda^\star_k} \pi_j \quad \forall i. \label{eq:optimality}
\end{align}
Comparing~\eqref{eq:optimality} to~\eqref{eq:balance}, it is clear that $\bm{\lambda}^\star$ is a solution of the steady-state inversion problem.
As such, the network choice model presented in this paper can be viewed as a principled extension of the steady-state inversion problem to the finite-data case.

%%%%%%%%%%%%%%%%%%%%%%%%%%%%%%%%%%%%%%%%%%%%%%%%%%%%%%%%%%%%%%%%%%%%%%%%%
\subsection{Maximum-Likelihood Estimate}
\label{sec:maxlik}

The log-likelihood~\eqref{eq:loglik} is not concave in $\bm{\lambda}$, but it can be made concave using the simple reparametrization $\lambda_i = e^{\theta_i}$.
Therefore, any local minimum of the likelihood is a global minimum.
Unfortunately, it turns out that the conditions guaranteeing that the ML estimate is well-defined (i.e., that it exists and is unique) are restrictive and impractical.
We illustrate this by providing a necessary condition, and for brevity we defer the comprehensive analysis of the ML estimate to Appendix~\ref{app:maxlik}.
We begin with a definition that uses the notion of \emph{hypergraph}, a generalized graph where edges may be any non-empty subset of nodes.
\begin{definition}[Comparison hypergraph]
Given a directed graph $G = (V, E)$, the \emph{comparison hypergraph} is the hypergraph $H = (V, \mathcal{A})$, with $\mathcal{A} = \{ N^+_i \mid i \in V \}$.
\end{definition}
Intuitively, $H$ is the hypergraph induced by the sets of alternatives available at each node.
Figure~\ref{fig:samplehyp} provides an example of a graph and of its associated comparison hypergraph.
Equipped with this definition, we can state the following theorem that is a reformulation of a well-known result for Luce's choice model \citep{hunter2004mm}.
\begin{theorem}
\label{thm:mlnecessary}
If the comparison hypergraph is not connected, then for any data $\mathcal{D}$ there are $\bm{\lambda}$ and $\bm{\mu}$ such that $\bm{\lambda} \neq c \bm{\mu}$ for any $c \in \mathbf{R}_{>0}$ and $\ell(\bm{\lambda} ; \mathcal{D}) = \ell(\bm{\mu} ; \mathcal{D}).$
\end{theorem}
In short, the proof shows that rescaling all the parameters in one of the connected components does not change the value of the likelihood function.
The network of Figure~\ref{fig:samplenet} illustrates an instance where the condition fails:
although the graph $G$ is strongly connected, its associated comparison hypergraph $H$ (depicted in Figure~\ref{fig:samplehyp}) is disconnected, and no matter what the data $\mathcal{D}$ is, the ML estimate will never be uniquely defined.
In fact, in Appendix~\ref{app:maxlik}, we demonstrate that Theorem~\ref{thm:mlnecessary} is just the tip of the iceberg.
We provide an example where the ML estimate \emph{does not exist} even though the comparison hypergraph is connected, and we explain that verifying a necessary and sufficient condition for the existence of the ML estimate is computationally \emph{more expensive} than solving the inference problem itself.
%All in all, these shortcomings drive us to seek a more robust estimator.

\begin{figure}[t]
  \centering
  \includegraphics[width=.5\linewidth]{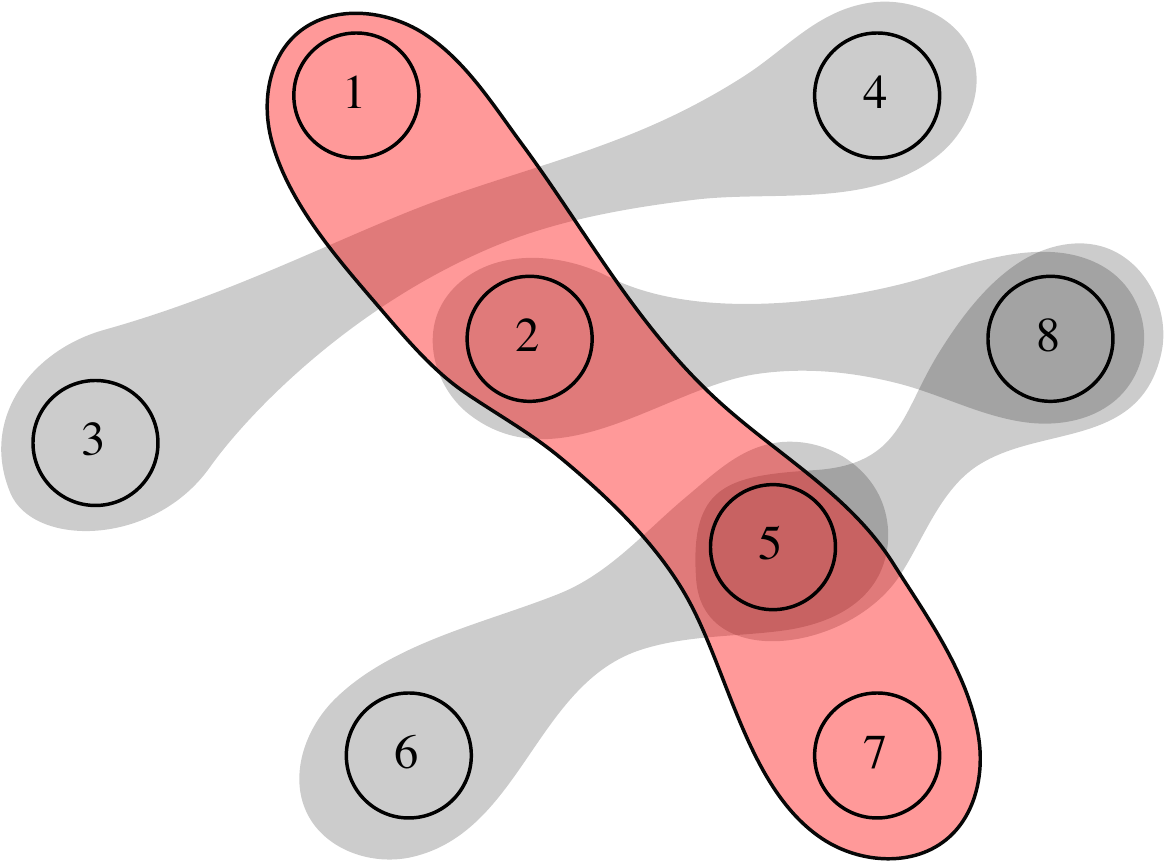}
  \caption{The comparison hypergraph associated to the network of Fig.~\ref{fig:samplenet}.
The hyperedge associated to $N^+_6$ is highlighted in red.
Note that the component $\{3, 4\}$ is disconnected from the rest of the hypergraph.}
  \label{fig:samplehyp}
\end{figure}

%%%%%%%%%%%%%%%%%%%%%%%%%%%%%%%%%%%%%%%%%%%%%%%%%%%%%%%%%%%%%%%%%%%%%%%%%
\subsection{Well-Posed Inference}
% https://en.wikipedia.org/wiki/Well-posed_problem

Following the ideas of \citet{caron2012efficient}, we introduce an independent Gamma prior on each parameter, i.e., i.i.d. $\lambda_1, \ldots, \lambda_n \sim \text{Gamma}(\alpha, \beta)$.
Adding the log-prior to the log-likelihood, we can write the log-posterior as
\begin{align}
\begin{aligned}
\label{eq:logpost}
\log p(\bm{\lambda} \mid \mathcal{D}) = \sum_{i = 1}^n
  \bigg[ (c^-_i + \alpha - 1) \log \lambda_i
        - c^+_i \log\!\sum_{k \in N^+_i}\!\lambda_k  - \beta \lambda_i \bigg] + \kappa,
\end{aligned}
\end{align}
where $\kappa$ is a constant that is independent of $\bm{\lambda}$.
The Gamma prior translates into a form of regularization that makes the inference problem well-posed, as shown by the following theorem.

\begin{theorem}
\label{thm:map}
If i.i.d. $\lambda_1, \ldots, \lambda_n \sim \text{Gamma}(\alpha, \beta)$ with $\alpha > 1$, then the log-posterior~\eqref{eq:logpost} always has a unique maximizer $\bm{\lambda}^\star \in \mathbf{R}^n_{>0}$.
\end{theorem}

The condition $\alpha > 1$ ensures that the prior has a nonzero mode.
In short, the proof of Theorem~\ref{thm:map} shows that as a result of the Gamma prior, the log-posterior can be reparametrized into a strictly concave function with bounded super-level sets (if $\alpha > 1$).
This guarantees that the log-posterior will always have exactly one maximizer.
Unlike the results that we derive for the ML estimate, Theorem~\ref{thm:map} does not impose any condition on the graph $G$ for the estimate to be well-defined.

\paragraph{Remark.}
Note that varying the rate $\beta$ in the Gamma prior simply rescales the parameters $\bm{\lambda}$.
Furthermore, it is clear from~\eqref{eq:singlelik} that such a rescaling affects neither the likelihood of the observed data nor the prediction of future transitions.
As a consequence, we may assume that $\beta = 1$ without loss of generality.

%% file: 05-algorithm.tex
%%%%%%%%%%%%%%%%%%%%%%%%%%%%%%%%%%%%%%%%%%%%%%%%%%%%%%%%%%%%%%%%%%%%%%%%%
\section{Inference Algorithm}  %%%%%%%%%%%%%%%%%%%%%%%%%%%%%%%%%%%%%%%%%%
\label{sec:algorithm}

The maximizer of the log-posterior does not have a closed-form solution.
In the spirit of the algorithms of \citet{hunter2004mm} for variants of Luce's choice model, we develop a minorization-maximization (MM) algorithm.
Simply put, the algorithm iteratively refines an estimate of the maximizer by solving a sequence of surrogates of the log-posterior.
Using the inequality $\log x \le \log \tilde{x} + x/\tilde{x} - 1$ (with equality if and only if $x = \tilde{x}$), we can lower-bound the log-posterior~\eqref{eq:logpost} by
\begin{align*}
&f^{(t)}(\bm{\lambda}) =
    \sum_{i = 1}^n \bigg[ (c^-_i + \alpha - 1) \log \lambda_i 
                         - c^+_i \bigg( \log\!\sum_{k \in N^+_i}\!\lambda^{(t)}_k
                                       +\frac{\sum_{k \in N^+_i}\!\lambda_k}{\sum_{k \in N^+_i}\!\lambda^{(t)}_k} -1 \bigg)
                         - \beta \lambda_i \bigg] + \kappa,
\end{align*}
with equality if and only if $\bm{\lambda} = \bm{\lambda}^{(t)}$.
Starting with an arbitrary $\bm{\lambda}^{(0)} \in \mathbf{R}^n_{>0}$, we repeatedly solve the optimization problem
\begin{align*}
\bm{\lambda}^{(t+1)} = \argmax_{\bm{\lambda}} f^{(t)}(\bm{\lambda}).
\end{align*}
Unlike the maximization of the log-posterior, the surrogate optimization problem has a closed-form solution, obtained by setting $\nabla f^{(t)}$ to $\bm{0}$:
\begin{align}
\label{eq:mmiter}
\lambda_i^{(t + 1)} = \frac{c^-_i + \alpha - 1}{\sum_{j \in N^-_i} \gamma_j^{(t)} + \beta},
\ \gamma_j^{(t)} = \frac{c^+_j}{\sum_{k \in N^+_j} \lambda_k^{(t)}}.
\end{align}
The iterates provably converge to the maximizer of~\eqref{eq:logpost}, as shown by the following theorem.

\begin{theorem}
\label{thm:mmconv}
Let $\bm{\lambda}^\star$ be the unique maximum a-posteriori estimate.
Then for any initial $\bm{\lambda}^{(0)} \in \mathbf{R}^n_{> 0}$ the sequence of iterates defined by~\eqref{eq:mmiter} converges to $\bm{\lambda}^\star$.
\end{theorem}

Theorem~\ref{thm:mmconv} follows from a standard result on the convergence of MM algorithms and uses the fact that the log-posterior increases after each iteration.
Furthermore, it is known that MM algorithms exhibit geometric convergence in a neighborhood of the maximizer \citep{lange2000optimization}.
A thorough investigation of the convergence properties is left for future work.

\begin{algorithm}[t]
  \caption{ChoiceRank}
  \label{alg:choicerank}
  \begin{algorithmic}[1]
    \Require graph $G = (V, E)$, counts $\{ (c^-_i, c^+_i) \}$
    \State $\bm{\lambda} \gets [1, \ldots, 1]$
    \Repeat
      \State $\bm{z} \gets \bm{0}_n$
      \Comment{Recompute $\bm{\gamma}$}
      \OneLineFor{$(i, j) \in E$} $z_i \gets z_i + \lambda_j$
      \OneLineFor{$i \in V$} $\gamma_i \gets c^+_i / z_i$
      \State $\bm{z} \gets \bm{0}_n$
      \Comment{Recompute $\bm{\lambda}$}
      \OneLineFor{$(i, j) \in E$} $z_j \gets z_j + \gamma_i$
      \OneLineFor{$i \in V$} $\lambda_i \gets (c^-_i + \alpha - 1) / (z_i + \beta)$
    \Until $\bm{\lambda}$ has converged
  \end{algorithmic}
\end{algorithm}
%In practice, we notice that adding a little bit of regularization through the Gamma prior greatly improves convergence.

The structure of the updates in~\eqref{eq:mmiter} leads to an extremely simple and efficient implementation, given in Algorithm~\ref{alg:choicerank}: we call it ChoiceRank.
A graphical representation of an iteration from the perspective of a single node is given in Figure~\ref{fig:msgpassing}.
Each iteration consists of two phases of message passing, with $\gamma_i$ flowing towards in-neighbors $N^-_i$, then $\lambda_i$ flowing towards out-neighbors $N^+_i$.
The updates to a node's state are a function of the sum of the messages.
As the algorithm does two passes over the edges and two passes over the vertices, an iteration takes $O(m + n)$ time.
The edges can be processed in any order, and the algorithm maintains a state over only $O(n)$ values associated with the vertices.
Furthermore, the algorithm can be conveniently expressed in the well-known vertex-centric programming model \citep{malewicz2010pregel}.
This makes it easy to implement ChoiceRank inside scalable, optimized graph-processing systems such as Apache Spark \citep{gonzalez2014graphx}.

\begin{figure}[t]
  \centering
  \includegraphics[width=0.5\linewidth]{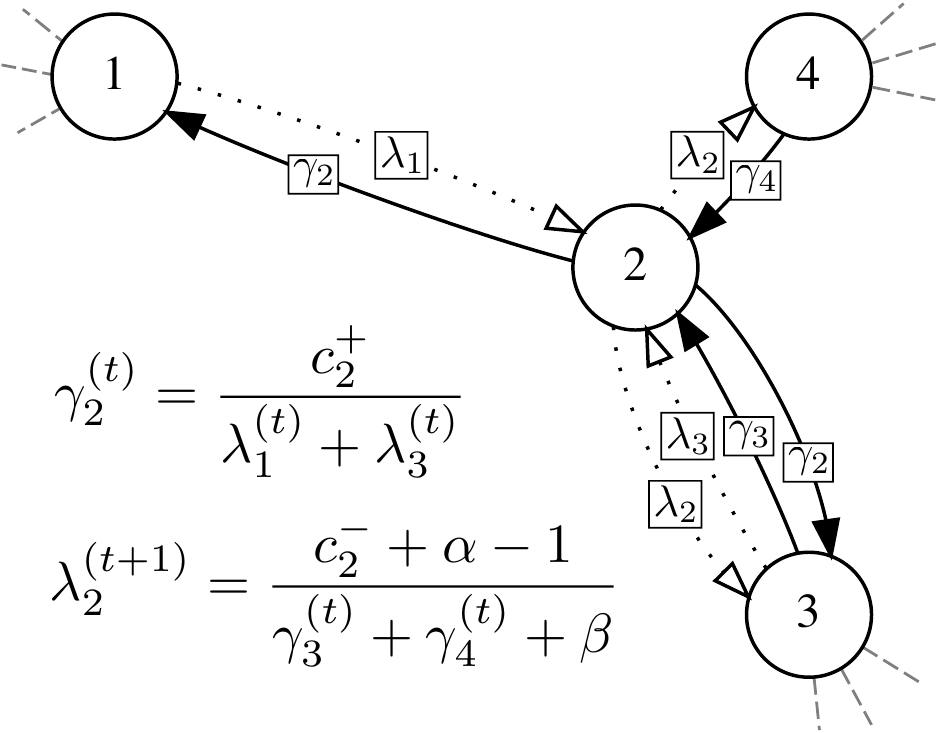}
  \caption{One iteration of ChoiceRank from the perspective of node $2$.
  Messages flow in both directions along the edges of the graph $G$, first in the reverse direction (in dotted) then in the forward direction (in solid).}
  \label{fig:msgpassing}
\end{figure}

\paragraph{EM viewpoint.}
The update~\eqref{eq:mmiter} can also be explained from an expectation-maximization (EM) viewpoint, by introducing suitable latent variables \citep{caron2012efficient}.
This viewpoint enables a Gibbs sampler that can be used for Bayesian inference.
We present the EM derivation in Appendix~\ref{app:algorithm}, but leave a study of fully Bayesian inference in the network choice model for future work.

%% file: 06-experiments.tex
%%%%%%%%%%%%%%%%%%%%%%%%%%%%%%%%%%%%%%%%%%%%%%%%%%%%%%%%%%%%%%%%%%%%%%%%%
\section{Experimental Evaluation}  %%%%%%%%%%%%%%%%%%%%%%%%%%%%%%%%%%%%%%
\label{sec:experiments}

% Think of the difference between *predictive* and *explanatory*.
In this section, we investigate
\begin{enuminline}
\item the ability of the network choice model to accurately recover transitions in real-world scenarios, and
\item the potential of ChoiceRank to scale to very large networks.
\end{enuminline}

\subsection{Accuracy on Real-World Data}
\label{sec:accuracy}

We evaluate the network choice model on three datasets that are representative of two distinct application domains.
%The first dataset contains clickstream data from the English Wikipedia, i.e., traces of users' navigation on a Web site.
%The second dataset consists of the records of all trips made using New York City's bicycle-sharing service during the year 2015, i.e., mobility traces in a large city.
Each dataset can be represented as a set of transition counts $\{ c_{ij} \}$ on a directed graph $G = (V,E)$.
We aggregate the transition counts into marginal traffic data $\{ (c^-_i, c^+_i) \mid i \in V \}$ and fit a network choice model by using ChoiceRank.
We set $\alpha = 2.0$ and $\beta = 1.0$ (these small values simply guarantee the convergence of the algorithm) and declare convergence when $\lVert \bm{\lambda}^{(t)} - \bm{\lambda}^{(t-1)} \rVert_1 / n < 10^{-8}$.
Given $\bm{\lambda}$, we estimate transition probabilities using $p_{ij} \propto \lambda_j$ as given by \eqref{eq:singlelik}.
To the best of our knowledge, there is no other published method tackling the problem of estimating transition probabilities from marginal traffic data.
Therefore, we compare our method to three baselines based on simple heuristics.
\begin{description}[topsep=1ex,itemsep=0ex]
\item[Traffic] Transitions probabilities are proportional to the traffic of the target node: $q_{ij}^T \propto c_j^{-}$.
\item[PageRank] Transition probabilities are proportional to the PageRank score of the target node: $q_{ij}^P \propto \text{PR}_j$.
\item[Uniform] Any transition is equiprobable: $q_{ij}^U \propto 1$.
\end{description}
The four estimates are compared against ground-truth transition probabilities derived from the edge traffic data: $p_{ij}^\star \propto c_{ij}$.
We emphasize that although per-edge transition counts $\{c_{ij}\}$ are needed to \emph{evaluate} the accuracy of the network choice model (and the baselines), these counts are not necessary for \emph{learning} the model---per-node marginal counts are sufficient.

Given a node $i$, we measure the accuracy of a distribution $\bm{q}_i$ over outgoing transitions using two error metrics, the KL-divergence and the (normalized) rank displacement:
\begin{align*}
D_{\text{KL}}(\bm{p}_i^\star, \bm{q}_i) &= \sum_{j \in N^+_i} p^\star_{ij} \log \frac{p^\star_{ij}}{q_{ij}}, \\
D_{\text{FR}}(\bm{p}_i^\star, \bm{q}_i) &= \frac{1}{\vert N^+_i \vert^2} \sum_{j \in N^+_i} \vert \sigma^\star_i(j) - \hat{\sigma}_i(j) \vert,
\end{align*}
where $\sigma^\star_i$ (respectively $\hat{\sigma}_i$) is the ranking of elements in $N^+_i$ by decreasing order of $p^\star_{ij}$ (respectively $q_{ij}$).
We report the distribution of errors ``over choices'', i.e., the error at each node $i$ is weighted by the number of outgoing transitions $c^+_i$.

\subsubsection{Clickstream Data}

\paragraph{Wikipedia}
The Wikimedia Foundation has a long history of publicly sharing aggregate, page-level web traffic data\footnote{See: \url{https://stats.wikimedia.org/}.}.
Recently, it also released clickstream data from the English version of Wikipedia \citep{wulczyn2016wikipedia}, providing us with essential ground-truth transition-level data.
We consider a dataset that contains information, extracted from the server logs, about the traffic each page of the English Wikipedia received during the month of March 2016.
Each page's incoming traffic is grouped by HTTP referrer, i.e., by the page visited prior to the request.
We ignore the traffic generated by external Web sites such as search engines and keep only the internal traffic (\num{18}\% of the total traffic in the dataset).
In summary, we obtain counts of transitions on the hyperlink graph of English Wikipedia articles.
The graph contains $n = \num{2316032}$ nodes and $m = \num{13181698}$ edges, and we consider slightly over \num{1.2} billion transitions over the edges.
On this dataset, ChoiceRank converges after \num{795} iterations.

\paragraph{Kosarak}
We also consider a second clickstream dataset from a Hungarian online news portal\footnote{The data is publicly available at \url{http://fimi.ua.ac.be/data/}.}.
The data consists of $\num{7029013}$ transitions on a graph containing $n = 41001$ nodes and $m = \num{974560}$ edges.
ChoiceRank converges after \num{625} iterations.

\begin{figure*}[t]
  \centering
  \includegraphics[width=\linewidth]{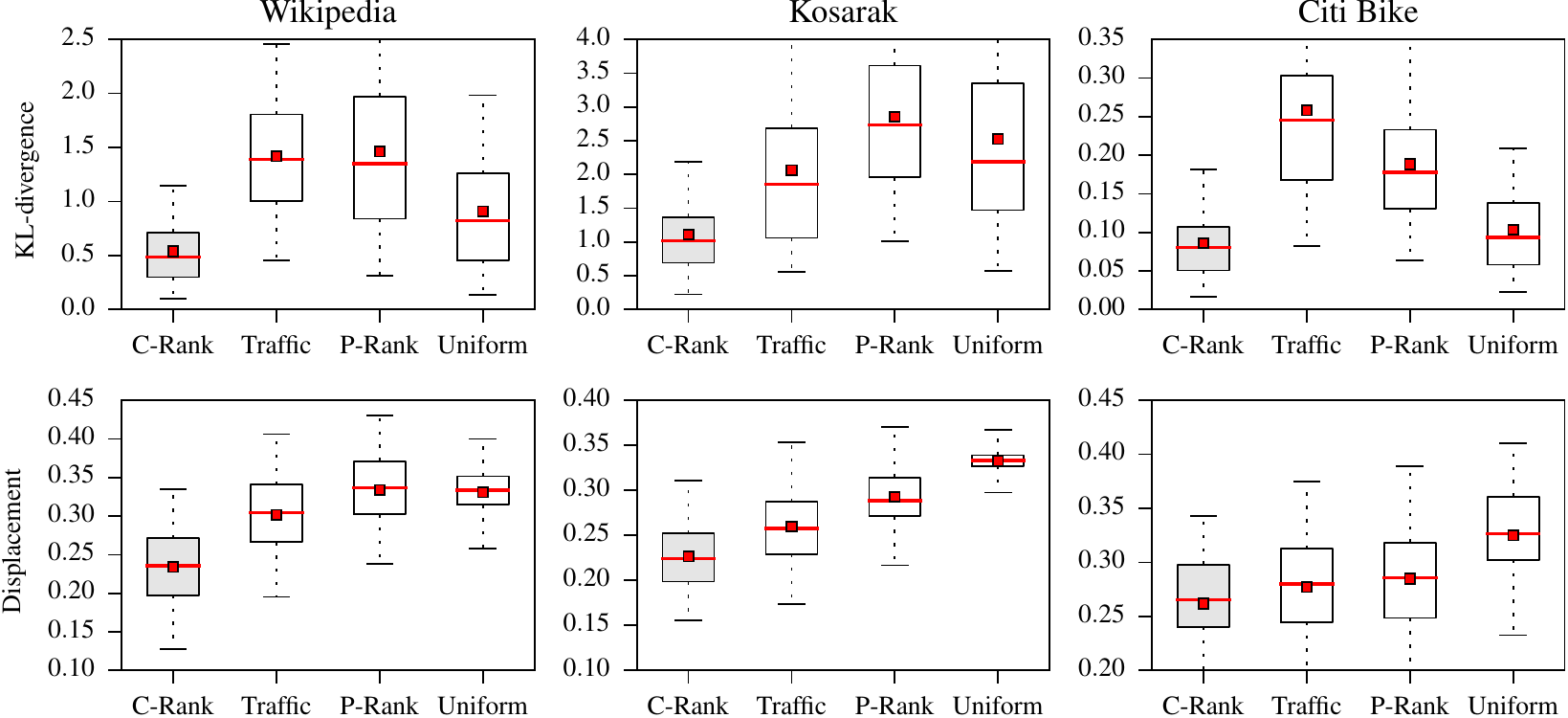}
  \caption{
Error distributions of the network choice model and three baselines for the Wikipedia (WP) and Citi Bike (CB) datasets.
The boxes show the interquartile range, the whiskers show the $5^{\text{th}}$ and $95^{\text{th}}$ percentiles, the red horizontal bars show the median and the red squares show the mean.
}
  \label{fig:results}
\end{figure*}

\begin{figure}[t]
  \centering
  \includegraphics[width=0.6\linewidth]{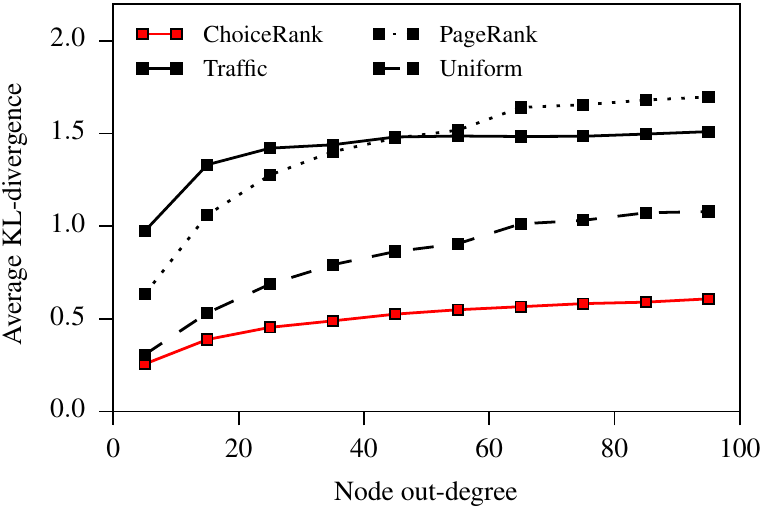}
  \caption{
Average KL-divergence as a function of the number of possible transitions for the Wikipedia dataset.
ChoiceRank performs comparatively better in the case where a node's out-degree is large.
}
  \label{fig:wpdegs}
\end{figure}

The four leftmost plots of Figure~\ref{fig:results} show the error distributions.
ChoiceRank significantly improves on the baselines, both in terms of KL-divergence and rank displacement.
These results give compelling evidence that transitions do not occur proportionally with the target's page traffic: in terms of KL-divergence, ChoiceRank improves on Traffic by a factor $3\times$ and $2\times$, respectively.
PageRank scores, while reflecting some notion of importance of a page, are not designed to estimate transitions, and understandably the corresponding baseline performs poorly.
Uniform (perhaps the simplest of our baselines) is (by design) unable to distinguish among transitions, resulting in a large displacement error.
We believe that its comparatively better performance in terms of KL-divergence (for Wikipedia) is mostly an artifact of the metric, which encourages ``prudent'' estimates.
Finally, in Figure~\ref{fig:wpdegs} we observe that ChoiceRank seems to perform comparatively better as the number of possible transition increases.

\subsubsection{NYC Bicycle-Sharing Data}

Next, we consider trip data from Citi Bike, New York City's bicycle-sharing system\footnote{The data is available at \url{https://www.citibikenyc.com/system-data}.}.
%Markov models have been used with success in the context of mobility prediction \citep{ashbrook2003using, kafsi2015traveling}. TODO
For each ride on the system made during the year 2015, we extract the pick-up and drop-off stations and the duration of the ride.
Because we want to focus on direct trips, we exclude rides that last more than one hour.
We also exclude source-destinations pairs which have less than 1 ride per day on average (a majority of source-destination pairs appears at least once in the dataset).
The resulting data consists of \num{3.4} million rides on a graph containing $n = \num{497}$ nodes and $m = \num{5209}$ edges.
ChoiceRank converges after $\num{7508}$ iterations.
We compute the error distribution in the same way as for the clickstream datasets.

The two rightmost plots of Figure~\ref{fig:results} display the results.
The observations made on the clickstream datasets carry over to this mobility dataset, albeit to a lesser degree.
A significant difference between clicking a link and taking a bicycle trip is that in the latter case, there is a non-uniform ``cost'' of a transition due to the distance between source and target.
In future work, one might consider incorporating edge weights and using the weighted network choice model presented in Appendix~\ref{app:extensions}.

\subsection{Scaling ChoiceRank to Billions of Nodes}

To demonstrate ChoiceRank's scalability, we develop a simple implementation in the Rust programming language, based on the ideas of COST \citep{mcsherry2015scalability}.
Our code is publicly available online\footnote{See: \url{http://lucas.maystre.ch/choicerank}.}.
The implementation repeatedly streams edges from disk and keeps four floating-point values per node in memory:
the counts $c^-_i$ and $c^+_i$, the sum of messages $z_i$, and either $\gamma_i$ or $\lambda_i$ (depending on the stage in the iteration).
As edges can be processed in any order, it can be beneficial to reorder the edges in a way that accelerates the computation.
For this reason, our implementation preprocesses the list of edges and reorders them in Hilbert curve order\footnote{A Hilbert space-filling curve visits all the entries of the adjacency matrix of the graph, in a way that preserves locality of both source and destination of the edges.}.
This results in better cache locality and yields a significant speedup.

We test our implementation on a hyperlink graph extracted from the 2012 Common Crawl web corpus\footnote{
The data is available at \url{http://webdatacommons.org/hyperlinkgraph/}.} that contains over \num{3.5} billion nodes and \num{128} billion edges \citep{meusel2014graph}.
The edge list alone requires about $1$ TB of uncompressed storage.
There is no publicly available information on the traffic at each page, therefore we generate a value $c_i$ for every node $i$ randomly and uniformly between \num{100} and \num{500}, and set both $c^-_i$ and $c^+_i$ to $c_i$.
As such, this experiment does not attempt to measure the validity of the model (unlike the experiments of Section~\ref{sec:accuracy}).
Instead, it focuses on testing the algorithm's potential to scale to to very large networks.

\paragraph{Results.}
We run \num{20} iterations of ChoiceRank on a dual Intel Xeon E5-2680 v3 machine, with \num{256} GB of RAM and \num{6} HDDs configured in RAID 0.
We arbitrarily set $\alpha = 2.0$ and $\beta = 1.0$ (but this choice has no impact on the results).
Only about \num{65} GB of memory is used, all to store the nodes' state ($4 \times 4$ bytes per node).
The algorithm takes a little less than \num{39} minutes per iteration on average.
Collectively, these results validate the feasibility of model inference for very large datasets.

It is worth noting that despite tackling different problems, the ChoiceRank algorithm exhibits interesting similarities with a message-passing implementation of PageRank commonly used in scalable graph-parallel systems such as Pregel \citep{malewicz2010pregel} and Spark \citep{gonzalez2014graphx}.
For comparison, using the COST code \citep{mcsherry2015scalability} we run \num{20} iterations of PageRank on the same hardware and data.
PageRank uses slightly less memory (about \num{50} GB, or one less floating-point number per node) and takes about half of the time per iteration (a little over \num{20} minutes).
This is consistent with the fact that ChoiceRank requires two passes over the edges per iteration, whereas PageRank requires one.
The similarities between the two algorithms lead us to believe that in general, ChoiceRank can benefit from any new system optimizations developed for PageRank.

%% file: 07-conclusion.tex
%%%%%%%%%%%%%%%%%%%%%%%%%%%%%%%%%%%%%%%%%%%%%%%%%%%%%%%%%%%%%%%%%%%%%%%%%
\section{Conclusion}  %%%%%%%%%%%%%%%%%%%%%%%%%%%%%%%%%%%%%%%%%%%%%%%%%%%
\label{sec:conclusion}

In this paper, we present a method that tackles the problem of finding the transition probabilities along the edges of a network, given only the network's structure and aggregate node-level traffic data.
This method generalizes and extends ideas recently presented by \citet{kumar2015inverting}.
We demonstrate that in spite of the strong model assumptions needed to learn $O(n^2)$ probabilities from $O(n)$ observations, the method still manages to recover the transition probabilities to a good level of accuracy on two clickstream datasets, and shows promise for applications beyond clickstream data.
To sum up, we believe that our method will be useful to pracitioners interested in understanding patterns of navigation in networks from aggregate traffic data, commonly available, e.g., in public datasets.

\paragraph{Acknowledgments.}
We thank Holly Cogliati-Bauereis, Ksenia Konyushkova, Brunella Spinelli and anonymous reviewers for careful proofreading and helpful comments.

%% file: 0A-extensions.tex
%%%%%%%%%%%%%%%%%%%%%%%%%%%%%%%%%%%%%%%%%%%%%%%%%%%%%%%%%%%%%%%%%%%%%%%%%
\section{Extensions and Proofs}  %%%%%%%%%%%%%%%%%%%%%%%%%%%%%%%%%%%%%%%%
\label{app:extensions}

In this section, we start by generalizing the network choice model to account for edge weights.
Then, we present formal proofs for
\begin{enuminline}
\item the (minimal) sufficiency of marginal counts and
\item the well-posedness of MAP inference
\end{enuminline}
in the generalized weighted network choice model.

%%%%%%%%%%%%%%%%%%%%%%%%%%%%%%%%%%%%%%%%%%%%%%%%%%%%%%%%%%%%%%%%%%%%%%%%%
\subsection{Generalization of the Model}

Let $G = (V, E)$ be a weighted, directed graph with edge weights $w_{ij} > 0$ for all $(i, j) \in E$.
\citet{kumar2015inverting} propose the following generalization of Luce's choice model.
Given a parameter vector $\bm{\lambda} \in \mathbf{R}_{>0}^n$, they define the choice probabilities as
\begin{align}
\label{eq:wsinglelik}
p_{ij} = \frac{w_{ij} \lambda_j}{\sum_{k \in N^+_i} w_{ik} \lambda_k}, \quad j \in N^+_i.
\end{align}
We refer to this model as the \emph{weighted network choice model}.
Intuitively, the strength of each alternative is weighted by the corresponding edge's weight;
Luce's original choice model is obtained by setting $w_{ij} = \text{constant}$.
In this general model, the log-likelihood becomes
\begin{align}
\ell(\bm{\lambda} ; \mathcal{D})
    &= \sum_{(i,j) \in E} c_{ij} \bigg[ \log w_{ij} \lambda_j - \log \sum_{k \in N^+_i} w_{ik} \lambda_k \bigg] \nonumber \\
    &= \sum_{(i,j) \in E} c_{ij} \bigg[ \log \lambda_j - \log \sum_{k \in N^+_i} w_{ik} \lambda_k \bigg] \nonumber
       + \sum_{(i,j) \in E} c_{ij} \log w_{ij}, \nonumber \\
    &= \sum_{i = 1}^n \bigg[ c^-_i \log \lambda_i - c^+_i \log\!\sum_{k \in N^+_i}\!w_{ik} \lambda_k \bigg] + \kappa_1, \label{eq:wloglik}
\end{align}
where $c^-_i = \sum_{j \in N^-_i} c_{ji}$ and $c^+_i = \sum_{j \in N^+_i} c_{ij}$ is the aggregate number of transitions arriving in and originating from $i$, respectively.
Note that for every $i$, the weights $\{ w_{ij} \mid j \in N^+_i \}$ are equivalent up to rescaling.

This generalization is relevant in situations where the current context modulates the alternatives' strength.
For example, this could be used to take into account the position or prominence of a link on a page in a hyperlink graph, or the distance between two locations in a mobility network.

%%%%%%%%%%%%%%%%%%%%%%%%%%%%%%%%%%%%%%%%%%%%%%%%%%%%%%%%%%%%%%%%%%%%%%%%%
\subsection{Minimal Sufficiency of Marginal Counts}

Recall that $c_{ij}$ denotes the number of times we observe a transition from $i$ to $j$.
We set out to prove the following theorem for the weighted network choice model.

\begin{theorem}
Let $c^-_i = \sum_{j \in N^-_i} c_{ji}$ and $c^+_i = \sum_{j \in N^+_i} c_{ij}$ be the aggregate number of transitions arriving in and originating from $i$, respectively.
Then, $\{ (c^-_i, c^+_i) \mid i \in V \}$ is a minimally sufficient statistic for the parameter $\bm{\lambda}$ in the weighted network choice model.
\end{theorem}

\begin{proof}
Let $f(\{ c_{ij} \} \mid \bm{\lambda})$ be the discrete probability density function of the data under the model with parameters $\bm{\lambda}$.
Theorem $6.2.13$ in \citet{casella2002statistical} states that $\{ (c^-_i, c^+_i) \}$ is a minimally sufficient statistic for $\bm{\lambda}$ if and only if, for any $\{ c_{ij} \}$ and $\{ d_{ij} \}$ in the support of $f$,
\begin{align}
\label{eq:minsuff}
\begin{aligned}
\frac{ f(\{ c_{ij} \} \mid \bm{\lambda}) }{ f(\{ d_{ij} \} \mid \bm{\lambda}) }\ \text{is independent of $\bm{\lambda}$}
\iff (c^-_i, c^+_i) = (d^-_i, d^+_i) \quad \forall i.
\end{aligned}
\end{align}
Taking the log of the ratio on the left-hand side and using~\eqref{eq:wloglik}, we find that
\begin{align*}
\log \frac{ f(\{ c_{ij} \} \mid \bm{\lambda}) }{ f(\{ d_{ij} \} \mid \bm{\lambda}) } =
  \sum_{i = 1}^n \bigg[ (c^-_i\!-\!d^-_i) \log \lambda_i
                       - (c^+_i\!-\!d^+_i) \log\!\sum_{k \in N^+_i}\!w_{ik} \lambda_k \bigg] + \kappa_2.
\end{align*}
From this, it is easy to see that the ratio of densities is independent of $\bm{\lambda}$ if and only if $c^-_i = d^-_i$ and $c^+_i = d^+_i$, which verifies~\eqref{eq:minsuff}.
\end{proof}

%%%%%%%%%%%%%%%%%%%%%%%%%%%%%%%%%%%%%%%%%%%%%%%%%%%%%%%%%%%%%%%%%%%%%%%%%
\subsection{Well-Posedness of MAP Inference}

Using a $\text{Gamma}(\alpha, \beta)$ prior for each parameter, the log-posterior of the weighted network choice model can be written as
\begin{align}
\label{eq:wlogpost}
\begin{aligned}
&\log p(\bm{\lambda} \mid \mathcal{D}) =
    \sum_{i = 1}^n \bigg[ (c^-_i + \alpha - 1) \log \lambda_i
        - c^+_i \log \bigg( \sum_{k \in N^+_i} w_{ik} \lambda_k \bigg)  - \beta \lambda_i \bigg]
    + \kappa_3.
\end{aligned}
\end{align}
We prove a theorem that guarantees that MAP estimation is well-posed in this generalized model; the proof of Theorem~\ref{thm:map} follows trivially.

\begin{theorem}
\label{thm:wmap}
If i.i.d. $\lambda_1, \ldots, \lambda_n \sim \text{Gamma}(\alpha, \beta)$ with $\alpha > 1$, then there exists a unique maximizer $\bm{\lambda}^\star \in \mathbf{R}^n_{>0}$ of the weighted network choice model's log-posterior~\eqref{eq:wlogpost}.
\end{theorem}

\begin{proof}
The log-posterior~\eqref{eq:wlogpost} is not concave in $\bm{\lambda}$, but it can be made concave using the simple reparametrization $\lambda_i = e^{\theta_i}$.
Under this reparametrization, the log-prior and the log-likelihood become
\begin{align*}
\log p(\bm{\theta})
    &= \sum_{i = 1}^n \left[ (\alpha - 1) \theta_i - \beta e^{\theta_i} \right] + \kappa_4, \\
\ell(\bm{\theta} ; \mathcal{D})
    &= \sum_{i = 1}^n \bigg[ c^-_i \theta_i - c^+_i \log \sum_{k \in N^+_i} w_{ik} e^{\theta_k} \bigg] + \kappa_5.
\end{align*}
It is easy to see that the log-likelihood is concave and the log-prior strictly concave in $\bm{\theta}$.
As a result, the log-posterior is strictly concave in $\bm{\theta}$, which ensures that there exists at most one maximizer.

Now consider any transition counts $\{ c_{ij} \}$ that satisfy $c^-_i = \sum_{j \in N^-_i} c_{ji}$ and $c^+_i = \sum_{j \in N^+_i} c_{ij}$.
The log-posterior can be written as
\begin{align*}
\log p(\bm{\theta} \mid \mathcal{D})
    &= \sum_{i = 1}^n \sum_{j \in N^+_i} c_{ij} \bigg[ \theta_j - \log \sum_{k \in N^+_i} w_{ik} e^{\theta_k} \bigg]
       + \sum_{i = 1}^n \left[ (\alpha - 1) \theta_i - \beta e^{\theta_i} \right] + \kappa_3\\
    &\le -n^2 \cdot \max_{i,j} \log w_{ij}
       + \sum_i^n \left[ (\alpha - 1) \theta_i - \beta e^{\theta_i} \right] + \kappa_3.
\end{align*}
For $\alpha > 1$, it follows that $\lim_{\lVert \bm{\theta} \rVert \to \infty} \log p(\bm{\theta} \mid \mathcal{D}) = -\infty$, which ensures that there is at least one maximizer.
\end{proof}

Note that Theorem~\ref{thm:wmap} can easily be extended to independent but non-identical Gamma priors, where $\lambda_i \sim \text{Gamma}(\alpha_i, \beta_i)$ and $\alpha_i \ne \alpha_j$, $\beta_i \ne \beta_j$ in general.

%% file: 0B-maxlik.tex
%%%%%%%%%%%%%%%%%%%%%%%%%%%%%%%%%%%%%%%%%%%%%%%%%%%%%%%%%%%%%%%%%%%%%%%%%
\section{Maximum-Likelihood Estimation}  %%%%%%%%%%%%%%%%%%%%%%%%%%%%%%%%
\label{app:maxlik}

In this section, we go into the analysis of the ML estimator in depth.
From the definition of choice probabilities in~\eqref{eq:wsinglelik}, it is clear that the likelihood is invariant to a rescaling of the parameters, i.e., $\ell(\bm{\lambda}; \mathcal{D}) = \ell(s \bm{\lambda}; \mathcal{D})$ for any $s > 0$.
We will therefore identify parameters up to rescaling.

%%%%%%%%%%%%%%%%%%%%%%%%%%%%%%%%%%%%%%%%%%%%%%%%%%%%%%%%%%%%%%%%%%%%%%%%%
\subsection{Necessary and Sufficient Conditions}

In order to provide a data-dependent, necessary and sufficient condition that guarantees that the ML estimate is well-defined, we extend the definition of comparison hypergraph presented in Section~\ref{sec:maxlik}.

\begin{definition}[Comparison graph]
Let $G = (V, E)$ be a directed graph and $\{ a_{ij} \mid (i,j) \in E \}$ be non-negative numbers.
The \emph{comparison graph} induced by $\{ a_{ij} \}$ is the directed graph $H = (V, E')$, where $(i,j) \in E'$ if and only if there is a node $k$ such that $i, j \in N^+_k$ and $a_{kj} > 0$.
\end{definition}

The numbers $\{ a_{ij}\}$ can be loosely interpreted as transition counts (although they do not need to be integer).
Intuitively, there is an edge $(i, j)$ in the comparison graph whenever there is at least one instance in which $i$ and $j$ were among the alternatives and $j$ was selected.
If $a_{ij} > 0$ for all edges, then the comparison graph is equivalent to its hypergraph counterpart, in that every hyperedge induces a clique in the comparison graph.
As shown by the next theorem, the notion of (data-dependent) comparison graph leads to a precise characterization of whether the ML estimate is well-defined or not.

\begin{theorem}
\label{thm:mlboth}
Let $G = (V, E)$ be a directed graph and $\{ (c^-_i, c^+_i) \}$ be the aggregate number of transitions arriving in and originating from $i$, respectively.
Let $\{ a_{ij} \}$ be any set of non-negative real numbers that satisfy
\begin{align*}
\sum_{j \in N^-_i} a_{ji} = c^-_i, \quad
\sum_{j \in N^+_i} a_{ij} = c^+_i.
\end{align*}
Then, the maximizer of the log-likelihood~\eqref{eq:wloglik} exists and is unique (up to rescaling) if and only if the comparison graph induced by $\{ a_{ij} \}$ is strongly connected.
\end{theorem}

The proof borrows from \citet{hunter2004mm}, in particular from the proofs of Lemmas~$1$ and~$2$.

\begin{proof}
The log-likelihood~\eqref{eq:wloglik} is not concave in $\bm{\lambda}$, but it can be made concave using the reparametrization $\lambda_i = e^{\theta_i}$.
We can rewrite the reparametrized log-likelihood using $\{ a_{ij} \}$ as
\begin{align*}
    \ell(\bm{\theta})
        = \sum_{i = 1}^n \sum_{j \in N^+_i} a_{ij} \bigg[ \theta_j - \log \sum_{k \in N^+_i} w_{ik} e^{\theta_k} \bigg],
\end{align*}
and, without loss of generality, we can assume that $\sum_i \theta_i = 0$ and $\min_{ij} w_{ij} = 1$.

First, we shall prove that the super-level set $\{ \bm{\theta} \mid \ell(\bm{\theta}) \ge c \}$ is bounded and compact for any $c$, if and only if the comparison graph is strongly connected.
The compactness of all super-level sets ensures that there is at least one maximizer.
Pick any unit vector $\bm{u}$ such that $\sum_i u_i = 0$, and let $\bm{\theta} = s \bm{u}$
When $s \to \infty$, then $e^{\theta_i} > 0$ and $e^{\theta_j} \to 0 $ for some $i$ and $j$.
As the comparison graph is strongly connected, there is a path from $i$ to $j$, and along this path there must be two consecutive nodes $i', j'$ such that $e^{\theta_{i'}} > 0$ and $e^{\theta_{j'}} \to 0$.
The existence of the edge $(i',j')$ in the comparison graph means that there is a $k$ such that $i', j' \in N^+_k$ and $a_{kj'} > 0$.
Therefore, the log-likelihood can be bounded as
\begin{align*}
\ell(\bm{\theta})
    \le a_{kj'} \bigg[ \theta_{j'} - \log \sum_{q \in N^+_k} w_{kq} e^{\theta_q} \bigg]
    \le a_{kj'} \left[ \theta_{j'} - \log (e^{\theta_{j'}} + e^{\theta_{i'}}) \right],
\end{align*}
and $\lim_{s \to \infty} \ell(\bm{\theta}) = -\infty$.
Conversely, suppose that the comparison graph is not strongly connected and partition the vertices into two non-empty subsets $S$ and $T$ such that there is no edge from $S$ to $T$.
Let $c > 0$ be any positive constant, and take $\tilde{\theta}_i = \theta_i + c$ if $i \in S$ and $\tilde{\theta}_i = \theta_i$ if $i \in T$ (renormalize such that $\sum_i \tilde{\theta}_i = 0$).
Clearly, $\ell(\tilde{\bm{\theta}}) \ge \ell(\bm{\theta})$, and by repeating this procedure $\lVert \bm{\theta} \rVert$ may be driven to infinity without decreasing the likelihood.

Second, we shall prove that if the comparison graph is strongly connected, the log-likelihood is strictly concave (in $\bm{\theta}$).
In particular, for any $p \in (0,1)$,
\begin{align}
\label{eq:strictconcav}
\ell \left[ p \bm{\theta} + (1-p) \bm{\eta} \right] \ge p \ell(\bm{\theta}) + (1-p) \ell(\bm{\eta}),
\end{align}
with equality if and only if $\bm{\theta} \equiv \bm{\eta}$ up to a constant shift.
Strict concavity ensures that there is at most one maximizer of log-likelihood.
We start with Hölder's inequality, which implies that, for positive $\{ x_k \}$ and $\{ y_k \}$, and $p \in (0,1)$,
\begin{align*}
\log \sum_k x_k^p y_k^{1-p} \le p \log \sum_k x_k + (1-p) \log \sum_k y_k.
\end{align*}
with equality if and only $x_k = c y_k$ for some $c > 0$.
Letting $x_k = w_{ik} e^{\theta_k}$ and $y_k = w_{ik} e^{\eta_k}$, we find that for all $i$
\begin{align}
\label{eq:holderapp}
\begin{aligned}
\log \sum_{k \in N^+_i} w_{ik} e^{p \theta_k + (1-p) \eta_k}
    \le p \log\!\sum_{k \in N^+_i}\!w_{ik} e^{\theta_k} + (1-p) \log\!\sum_{k \in N^+_i}\!w_{ik} e^{\eta_k},
\end{aligned}
\end{align}
with equality if and only if there exists $c \in \mathbf{R}$ such that $\theta_k = \eta_k + c$ for all $k \in N^+_{i}$.
Multiplying by $a_{ij}$ and summing over $i$ and $j$ on both sides of~\eqref{eq:holderapp} shows that the log-likelihood is concave in $\bm{\theta}$.
Now, consider any partition of the vertices into two non-empty subsets $S$ and $T$.
Because the comparison graph is strongly connected, there is always $k \in V$, $i \in S$ and $j \in T$ such that $i, j \in N^+_k$ and $a_{ki} > 0$.
Therefore, the left and right side of~\eqref{eq:strictconcav} are equal if and only if $\bm{\theta} \equiv \bm{\eta}$ up to a constant shift.

Bounded super-level sets and strict concavity form necessary and sufficient conditions for the existence and uniqueness of the maximizer.
\end{proof}

We now give a proof for Theorem~\ref{thm:mlnecessary}, presented in the main body of text.

\begin{proof}[Proof of Theorem~\ref{thm:mlnecessary}]
If the comparison hypergraph is disconnected, then for any data $\mathcal{D}$, the (data-induced) comparison graph is disconnected too.
Furthermore, the connected components of the comparison graph are subsets of those of the hypergraph.
Partition the vertices into two non-empty subsets $S$ and $T$ such that there is no hyperedge between $S$ to $T$ in the comparison hypergraph.
Let $A = \{ i \mid N^+_i \subset S \}$ and $B = \{ i \mid N^+_i \subset T \}$.
By construction of the comparison hypergraph, $A \cap B = \varnothing$ and $A \cup B = V$.
The log-likelihood can be therefore be rewritten as
\begin{align*}
\ell(\bm{\theta}) =
    \sum_{i \in A} \sum_{j \in N^+_i} a_{ij} \bigg[ \log \lambda_j - \log \sum_{k \in N^+_i} w_{ik} \lambda_k \bigg]
    + \sum_{i \in B} \sum_{j \in N^+_i} a_{ij} \bigg[ \log \lambda_j - \log \sum_{k \in N^+_i} w_{ik} \lambda_k \bigg].
\end{align*}
The sum over $A$ involves only parameters related to nodes in $S$, while the sum over $B$ involves only parameters related to nodes in $T$.
Because the likelihood is invariant to a rescaling of the parameters, it is easy to see that we can arbitrarily rescale the parameters of the vertices in either $S$ or $T$ without affecting the likelihood.
\end{proof}

\begin{figure*}[t]
  \begin{subfigure}{.33\textwidth}
    \centering
    \includegraphics[width=.85\linewidth]{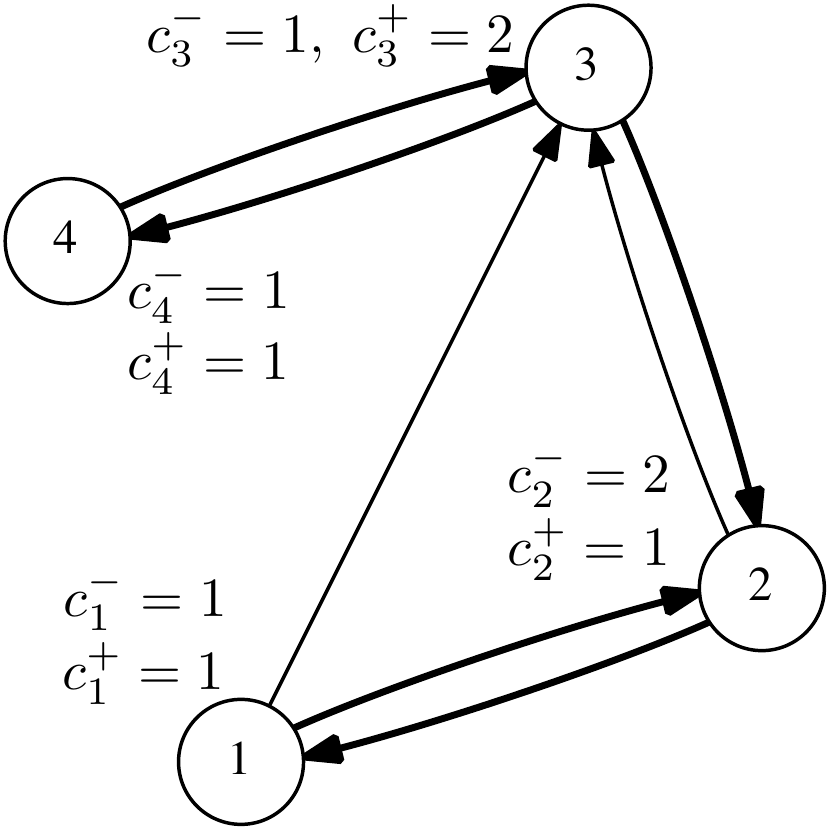}
    \caption{network structure}
  \end{subfigure}%
  \begin{subfigure}{.33\textwidth}
    \centering
    \includegraphics[width=.85\linewidth]{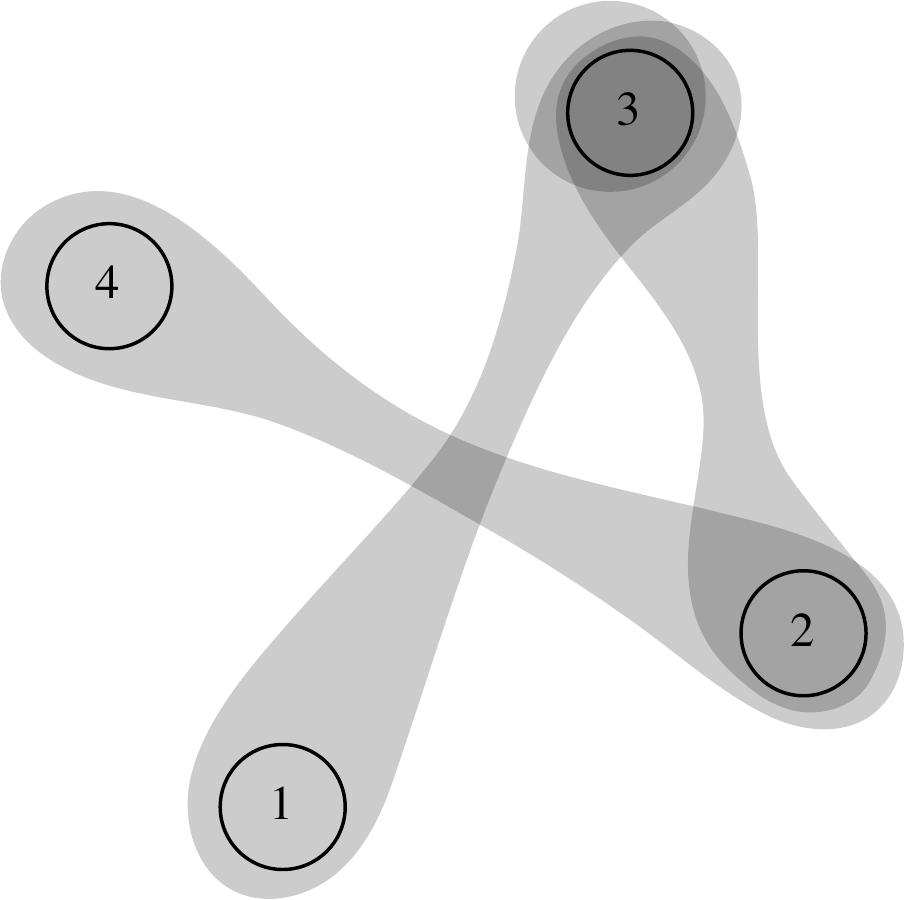}
    \caption{comparison hypergraph}
  \end{subfigure}
  \begin{subfigure}{.33\textwidth}
    \centering
    \includegraphics[width=.85\linewidth]{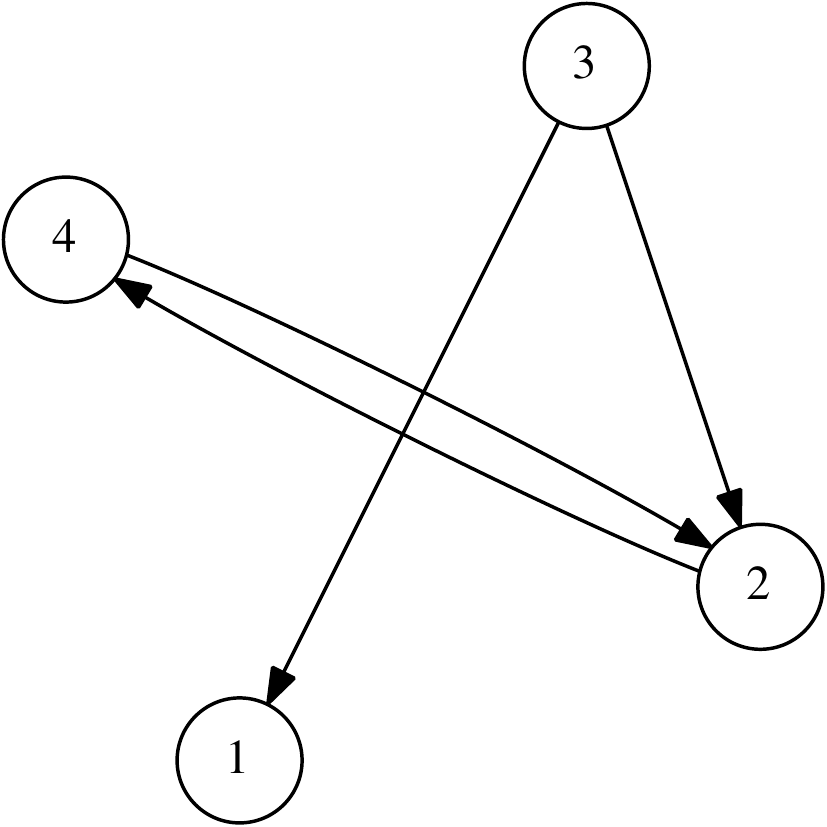}
    \caption{comparison graph}
  \end{subfigure}
  \caption{An innocent-looking example where the ML estimate does not exist.
  The network structure, aggregate traffic data and compatible transitions are shown on the left.
  While the comparison hypergraph is connected, the (data-dependent) comparison graph is not strongly connected.}
  \label{fig:badexample}
\end{figure*}

\paragraph{Verifying the condition of Theorem~\ref{thm:mlboth}.}
In order to verify the necessary and sufficient condition given $\{ (c^-_i, c^+_i) \}$, one has to find a non-negative solution $\{ a_{ij} \}$ to the system of equations
\begin{align*}
\sum_{j \in N^-_i} a_{ji} &= c^-_i, \\
\sum_{j \in N^+_i} a_{ij} &= c^+_i.
\end{align*}
\citet{dines1926positive} presents a remarkably simple algorithm to find such a non-negative solution.
Alternatively, \citet{kumar2015inverting} suggest recasting the problem as one of maximum flow in a network.
However, the computational cost of running \citeauthor{dines1926positive}' or max-flow algorithms is significantly greater than that of running ChoiceRank.

\subsection{Example}

To conclude our discussion, we provide an innocuous-looking example that highlights the difficulty of dealing with the ML estimate.
Consider the network structure and traffic data depicted in Figure~\ref{fig:badexample}.
The network is strongly connected, and its comparison hypergraph is connected as well; as such, the network satisfies the necessary condition stated in Theorem~\ref{thm:mlnecessary} in the main text.
Nevertheless, the condition is not sufficient for the ML-estimate to be well-defined.
In this example, the (data-dependent) comparison graph is \emph{not} strongly connected, and it is easy to see that the likelihood can always be increased by increasing $\lambda_1$, $\lambda_2$ and $\lambda_4$.
Hence, the ML estimate does not exist.

In this simple example, we indicate the edge transitions that generated the observed marginal traffic in bold.
Given this information, the comparison graph is easy to find, and the necessary and sufficient conditions of Theorem~\ref{thm:mlboth} are easy to check.
But in general, finding a set of transitions that is compatible with given marginal per-node traffic data is computationally expensive (see discussion above).

%% file: 0C-algorithm.tex
%%%%%%%%%%%%%%%%%%%%%%%%%%%%%%%%%%%%%%%%%%%%%%%%%%%%%%%%%%%%%%%%%%%%%%%%%
\section{ChoiceRank Algorithm}  %%%%%%%%%%%%%%%%%%%%%%%%%%%%%%%%%%%%%%%%%
\label{app:algorithm}

In this section, we start by generalizing the ChoiceRank algorithm to the weighted network choice model.
We then prove the convergence of this generalized algorithm.
Finally, we show how the same algorithm can be obtained from an EM viewpoint by introducing suitable latent variables.

%%%%%%%%%%%%%%%%%%%%%%%%%%%%%%%%%%%%%%%%%%%%%%%%%%%%%%%%%%%%%%%%%%%%%%%%%
\subsection{Algorithm for the Generalized Model}

Using the same linear upper-bound on the logarithm as in Section~\ref{sec:algorithm} of the main text, we can lower-bound the log-posterior~\eqref{eq:wlogpost} in the weighted model by
\begin{align}
\label{eq:wminorizing}
\begin{aligned}
f^{(t)}(\bm{\lambda}) = \kappa_2 + \sum_{i = 1}^n \bigg[
    & (c^-_i + \alpha - 1) \log \lambda_i - \beta \lambda_i \\
    &- c^+_i \bigg( \log\!\sum_{k \in N^+_i}\!w_{ik} \lambda^{(t)}_k
                   +\frac{\sum_{k \in N^+_i}\!w_{ik} \lambda_k}{\sum_{k \in N^+_i}\!w_{ik} \lambda^{(t)}_k} -1 \bigg) \bigg],
\end{aligned}
\end{align}
with equality if and only if $\bm{\lambda} = \bm{\lambda}^{(t)}$.
Starting with an arbitrary $\bm{\lambda}^{(0)} \in \mathbf{R}^n_{>0}$, we repeatedly maximize the lower-bound $f^{(t)}$.
This surrogate optimization problem has a closed form solution, obtained by setting $\nabla f^{(t)}$ to $0$:
\begin{align}
\label{eq:wmmiter}
\lambda_i^{(t + 1)} = \frac{c^-_i + \alpha - 1}{\sum_{j \in N^-_i} w_{ji} \gamma_j^{(t)} + \beta},
    \quad \text{where }
    \gamma_j^{(t)} = \frac{c^+_j}{\sum_{k \in N^+_j} w_{jk} \lambda_k^{(t)}}.
\end{align}
The iterates provably converge to the maximizer of~\eqref{eq:wlogpost}, as shown by the following theorem.

\begin{theorem}
\label{thm:wmmconv}
Let $\bm{\lambda}^\star$ be the unique maximum a-posteriori estimate.
Then for any initial $\bm{\lambda}^{(0)} \in \mathbf{R}^n_{> 0}$ the sequence of iterates defined by~\eqref{eq:wmmiter} converges to $\bm{\lambda}^\star$.
\end{theorem}

The proof follows that of \citeauthor{hunter2004mm}'s Theorem~$1$ \citeyearpar{hunter2004mm}.

\begin{proof}
Let $M: \mathbf{R}^n_{>0} \to \mathbf{R}^n_{>0}$ be the (continuous) map implicitly defined by one iteration of the algorithm.
For conciseness, let $g(\bm{\lambda}) \doteq \log p(\bm{\lambda} \mid \mathcal{D})$.
As $g$ has a unique maximizer and is concave using the reparametrization $\lambda_i = e^{\theta_i}$, it follows that $g$ has a single stationary point.
First, observe that the minorization-maximization property guarantees that $g \left[ M(\bm{\lambda}) \right] \ge g(\bm{\lambda})$.
Combined with the strict concavity of $g$, this ensures that $\lim_{t \to \infty} g(\bm{\lambda}^{(t)})$ exists and is unique for any $\bm{\lambda}^{(0)}$.
Second, $g \left[ M(\bm{\lambda}) \right] = g(\bm{\lambda})$ if and only if $\bm{\lambda}$ is a stationary point of $g$, because the minorizing function is tangent to $g$ at the current iterate.
It follows that $\lim_{t \to \infty} \bm{\lambda}^{(t)} = \bm{\lambda}^{\star}$.
\end{proof}

Theorem~\ref{thm:mmconv} of the main text follows directly by setting $w_{ij} \equiv 1$.
For completeness, the edge-streaming implementation adapted to the weighted model is given in Algorithm~\ref{alg:wchoicerank}.
The only changes with respect to Algorithm~\ref{alg:choicerank} (presented in the main text) are in lines~\ref{line:msg1} and~\ref{line:msg2}:
Every message $\gamma_i$ or $\lambda_j$ flowing through an edge $(i,j)$ is multiplied by the edge weight $w_{ij}$.

\begin{algorithm}[ht]
  \caption{ChoiceRank for the weighted model}
  \label{alg:wchoicerank}
  \begin{algorithmic}[1]
    \Require graph $G = (V, E)$, counts $\{ (c^-_i, c^+_i) \}$
    \State $\bm{\lambda} \gets [1, \ldots, 1]$
    \Repeat
      \State $\bm{z} \gets \bm{0}_n$
      \Comment{Recompute $\bm{\gamma}$}
      \OneLineFor{$(i, j) \in E$} $z_i \gets z_i + w_{ij} \lambda_j$ \label{line:msg1}
      \OneLineFor{$i \in V$} $\gamma_i \gets c^+_i / z_i$
      \State $\bm{z} \gets \bm{0}_n$
      \Comment{Recompute $\bm{\lambda}$}
      \OneLineFor{$(i, j) \in E$} $z_j \gets z_j + w_{ij} \gamma_i$ \label{line:msg2}
      \OneLineFor{$i \in V$} $\lambda_i \gets (c^-_i + \alpha - 1) / (z_i + \beta)$
    \Until $\bm{\lambda}$ has converged
  \end{algorithmic}
\end{algorithm}

%%%%%%%%%%%%%%%%%%%%%%%%%%%%%%%%%%%%%%%%%%%%%%%%%%%%%%%%%%%%%%%%%%%%%%%%%
\subsection{EM Viewpoint}

The MM algorithm can be seen from an EM viewpoint, following the ideas of \citet{caron2012efficient}.
We introduce $n$ independent random variables $\mathcal{Z} = \{ Z_i \mid i = 1, \ldots, n \}$, where
\begin{align*}
Z_i \sim \text{Gamma} \bigg( c^+_i, \sum_{j \in N^+_i} w_{ij} \lambda_j \bigg).
\end{align*}
With the addition of these latent random variables the complete log-likelihood becomes
\begin{align*}
\ell(\bm{\lambda} ; \mathcal{D}, \mathcal{Z})
    &= \ell(\bm{\lambda}, \mathcal{D}) + \sum_{i = 1}^n \log p(z_i \mid \mathcal{D}, \bm{\lambda}) \\
    &= \sum_{i = 1}^n \bigg[ c^-_i \log \lambda_i - c^+_i \log \sum_{k \in N^+_i} w_{ik} \lambda_k \bigg] \\
    &\qquad +\sum_{i = 1}^n \bigg[  c^+_i \log \sum_{k \in N^+_i} w_{ik} \lambda_k - z_i \sum_{k \in N^+_i} w_{ik} \lambda_k \bigg] + \kappa_6 \\
    &= \sum_{i = 1}^n \bigg[ c^-_i \log \lambda_i - z_i \sum_{k \in N^+_i} w_{ik} \lambda_k \bigg] + \kappa_6.
\end{align*}
Using a $\text{Gamma}(\alpha, \beta)$ prior for each parameter, the expected value of the log-posterior with respect to the conditional $\mathcal{Z} \mid \mathcal{D}$ under the estimate $\bm{\lambda}^{(t)}$ is
\begin{align*}
Q(\bm{\lambda}, \bm{\lambda}^{(t)})
    &= \mathbf{E}_{\mathcal{Z} \mid \mathcal{D}, \bm{\lambda}^{(t)}} \left[ \ell(\bm{\lambda} ; \mathcal{D}, \mathcal{Z}) \right]
       + \log p(\bm{\lambda}) \\
    &=\sum_{i = 1}^n \bigg[ c^-_i \log \lambda_i - c^+_i \frac{\sum_{k \in N^+_i} w_{ik} \lambda_k}{\sum_{k \in N^+_i} w_{ik} \lambda^{(t)}_k} \bigg]
      + \sum_{i = 1}^n \bigg[ (\alpha -1) \log \lambda_i - \beta \lambda_i \bigg] + \kappa_7
\end{align*}
The EM algorithm starts with an initial $\bm{\lambda}^{(0)}$ and iteratively refines the estimate by solving the optimization problem $\bm{\lambda}^{(t+1)} = \argmax_{\bm{\lambda}} Q(\bm{\lambda}, \bm{\lambda}^{(t)})$.
It is not difficult to see that for a given $\bm{\lambda}^{(t)}$, maximizing $Q(\bm{\lambda}, \bm{\lambda}^{(t)})$ is equivalent to maximizing the minorizing function $f^{(t)}(\bm{\lambda})$ defined in~\eqref{eq:wminorizing}.
Hence, the MM and the EM viewpoint lead to the exact same sequence of iterates.

The EM formulation leads to a Gibbs sampler in a relatively straightforward way \citep{caron2012efficient}.
We leave a systematic treatment of Bayesian inference in the network choice model for future work.